\documentclass[letterpaper]{article} 
\usepackage{aaai24}  
\usepackage{times}  
\usepackage{helvet}  
\usepackage{courier}  
\usepackage[hyphens]{url}  
\usepackage{graphicx} 
\usepackage{natbib}  
\usepackage{caption} 
\usepackage{algorithm}
\usepackage{algorithmic}
\usepackage{amsmath}
\newtheorem{theorem}{Theorem}
\newenvironment{proof}{{\noindent\it Proof.}\quad}

\usepackage{newfloat}
\usepackage{listings}
\DeclareCaptionStyle{ruled}{labelfont=normalfont,labelsep=colon,strut=off} 
\floatstyle{ruled}
\newfloat{listing}{tb}{lst}{}
\floatname{listing}{Listing}
\usepackage{color}
\usepackage[table,xcdraw]{xcolor}
\usepackage{multirow}
\usepackage[titletoc]{appendix}

\title{Memory-Efficient Reversible Spiking Neural Networks}
\author {
    Hong Zhang\textsuperscript{\rm 1},
    Yu Zhang\textsuperscript{\rm 1,\rm 2}\thanks{Corresponding Author.}
}
\affiliations {
    \textsuperscript{\rm 1}State Key Laboratory of Industrial Control Technology, College of Control Science and Engineering, Zhejiang University, Hangzhou, China\\
    \textsuperscript{\rm 2}Key Laboratory of Collaborative sensing and autonomous unmanned systems of Zhejiang Province, Hangzhou, China\\
    \{hongzhang99, zhangyu80\}@zju.edu.cn
}

\usepackage{bibentry}

\begin{document}

\maketitle

\begin{abstract}
Spiking neural networks (SNNs) are potential competitors to artificial neural networks (ANNs) due to their high energy-efficiency on neuromorphic hardware. However, SNNs are unfolded over simulation time steps during the training process. Thus, SNNs require much more memory than ANNs, which impedes the training of deeper SNN models. In this paper, we propose the reversible spiking neural network to reduce the memory cost of intermediate activations and membrane potentials during training. Firstly, we extend the reversible architecture along temporal dimension and propose the reversible spiking block, which can reconstruct the computational graph and recompute all intermediate variables in forward pass with a reverse process. On this basis, we adopt the state-of-the-art SNN models to the reversible variants, namely reversible spiking ResNet (RevSResNet) and reversible spiking transformer (RevSFormer). Through experiments on static and neuromorphic datasets, we demonstrate that the memory cost per image of our reversible SNNs does not increase with the network depth. On CIFAR10 and CIFAR100 datasets, our RevSResNet37 and RevSFormer-4-384 achieve comparable accuracies and consume $3.79\times$ and $3.00\times$ lower GPU memory per image than their counterparts with roughly identical model complexity and parameters. We believe that this work can unleash the memory constraints in SNN training and pave the way for training extremely large and deep SNNs. The code is available at https://github.com/mi804/RevSNN.git.

\end{abstract}

\section{Introduction}

\begin{figure}[htp]
\centering
\includegraphics[width=0.99\columnwidth]{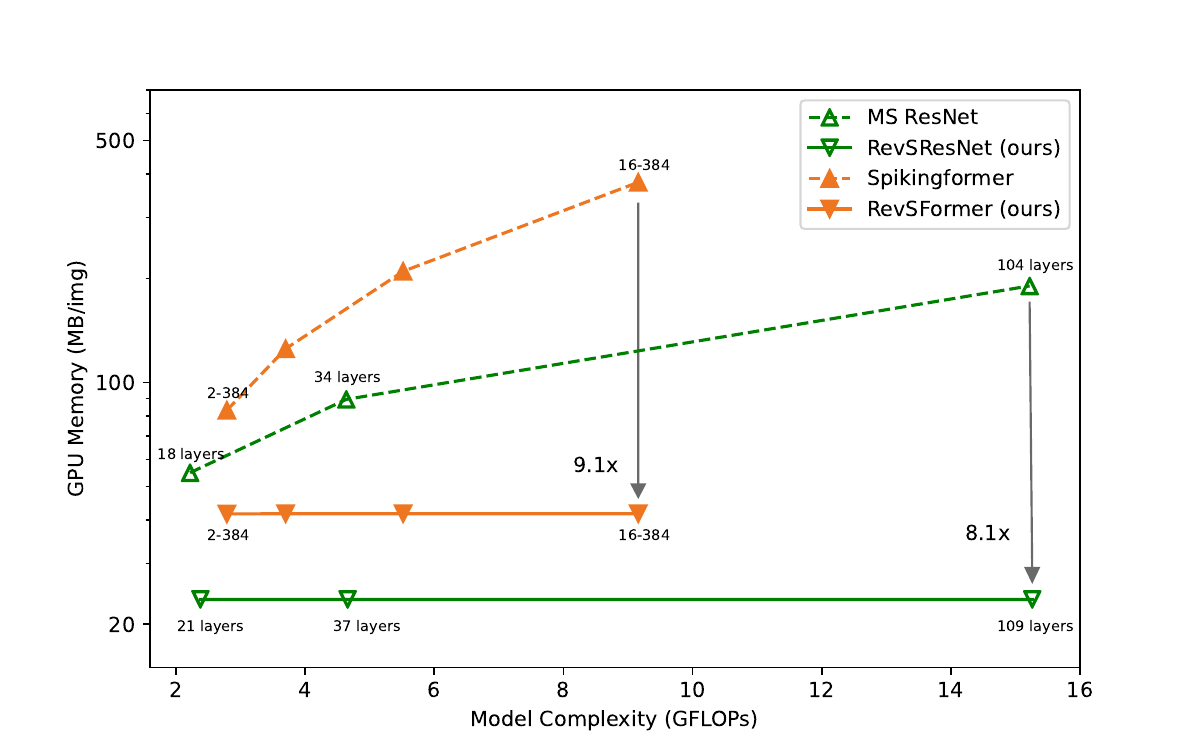} 
\caption{\small{Reversible spiking neural networks are more memory-efficient. Our proposed RevSResNet and RevSFormer need far less GPU memory per image than their non-reversible counterparts during training. Besides, the memory cost of reversible SNNs does not increase with the network depth.}
}
\label{gpu_cost}
\end{figure}

Spiking neural networks (SNNs), brain-inspired models based on binary spiking signals, are regarded as the third generation of neural networks \cite{maass1997networks}. Due to the sparsity and event-driven characteristics, SNNs can be deployed on neuromorphic hardware with low energy consumption. With the help of backpropagation through time framework (BPTT) and surrogate gradient, direct training SNNs are developing towards deeper and larger models. Advanced spiking architectures such as ResNet-like SNNs \cite{hu2021msresnet, sew, SGResNet} and spiking vision transformers \cite{zhou2022spikformer, zhou2023spikingformer} have been proposed in succession, indicating that SNNs are potential competitors to artificial neural networks (ANNs).

Although the inference process of SNNs on neuromorphic chips is relatively mature \cite{davies2018loihi, roy2019towards}, these chips cannot support the training process. Therefore, SNNs are still trained on graphics processing units (GPUs). During the training process based on the BPTT framework, SNNs are unfolded over simulation time steps $T$. Thus, SNNs usually require higher computing resources and memory bandwidth compared to ANNs. The computational requirements can be compensated by some AI accelerators \cite {cupy_learningsys2017} or spending more time in training. However, there is currently no solution to the memory constraints. Under such constraints, some SNNs are trained with a small batch size \cite{scgnet}, indirectly affecting the final accuracy \cite{wu2021rethinking}. Also, deeper SNNs are prevented from training.

The high memory consumption of SNNs comes from several aspects. On the one hand, like ANNs, the memory required by SNNs increases linearly with the depth of the network. The deeper the network, the more parameters and intermediate activations need storage. On the other hand, unlike ANNs, the memory cost of SNNs increases with simulation time step $T$. SNNs need to store $T$ times more intermediate activations, and the membrane potentials of spiking neurons need also to be stored for gradient computation. It is evident that a significant amount of memory consumption comes from storing intermediate activations and membrane potentials \cite{gomez2017reversible}. By reducing this part of the consumption, we can decouple the memory growth from the depth to a large extent.

In this work, we propose reversible spiking neural networks to reduce the memory cost of SNN training. The intention of reversibility is that each layer's input variables and membrane potentials can be re-computed by its output variables. Therefore, even if no intermediate variables are stored, we can quickly reconstruct them through such reversible transformation. In this work, we first extend the reversible architecture \cite{gomez2017reversible} along the temporal dimension to adapt to the BPTT training framework. On this basis, we propose spiking reversible block, which is reversible along spatial dimension and consistent along temporal dimension. Then, we present the reversible spiking ResNet (RevSResNet) and reversible spiking transformer (RevSFormer), which are the reversible counterparts of MS ResNet \cite{hu2021msresnet} and Spikingformer \cite{zhou2022spikformer} (the latest ResNet-like and transformer-like SNNs). As is shown in Figure~\ref{gpu_cost}, our networks consume much less memory per image than their counterparts. We verify the effect of RevSResNet and RevSFormer on static datasets (CIFAR10 and CIFAR100 \cite{cifar}) and neuromorphic datasets (CIFAR10-DVS \cite{li2017dvscifar10} and DVS128 Gesture \cite{dvs_gesture}). The experiments show that RevSResNet and RevSFormer have competitive performance to their non-reversible counterparts. At the same time, our reversible models significantly reduce memory cost during the training process, saving $3.79\times$ on the RevSResNet37 and $3.00\times$ on the RevSFormer-4-384 model.

In summary, our contributions are three-fold.
\begin{itemize}
    \item We analyze the reversibility of SNNs in the spatial and temporal dimensions and propose spiking reversible block for the BPTT framework. On this basis, each block's input and intermediate variables can be calculated by its outputs.
    \item We propose the reversible spiking ResNet (RevSResNet) and reversible spiking transformer (RevSFormer). We redesign a series of structures (such as downsample layers, reversible spiking residual block, and reversible spiking transformer block) to match the performance of the non-reversible state-of-the-art spiking counterparts.
    \item The experimens show that RevSResNet and RevSFormer have competitive performance to their non-reversible counterparts. At the same time, our reversible models significantly reduce memory cost during the training process.
\end{itemize}

\section{Related Works}
\subsection{Spiking Neural Networks}
SNNs utilize binary spikes to transmit and compute information, while the spiking neurons \cite{LIF, yao2022glif} play a crucial role in converting analog membrane potentials into binary spikes. There are two methods to obtain deep SNNs: ANN-to-SNN conversion and direct training. The ANN-to-SNN conversion methods \cite{fast_classify, bu2022optimized, deng2021optimal, wang2022signed} convert the same structured ANNs into SNNs, which usually achieves high accuracy. However, this method is limited because the obtained SNN requires a large time step and is unable to handle neuromorphic data. The direct training method utilizes error backpropagation to train SNNs directly, where the BPTT framework \cite{shrestha2018slayer} and surrogate gradient \cite{neftci2019surrogate} techniques play a vital role. In recent years, direct training spiking structures have been proposed successively, including ResNet-like models \cite{spikingresnet, sew, hu2021msresnet, SGResNet}, Spiking transformers \cite{zhou2022spikformer, zhou2023spikingformer}, NAS SNNs \cite{na2022autosnn, kim2022neural}, etc. These networks have lower latency, but the training process requires more computing resources and memory costs than ANNs. Among them, high memory cost limits the depth and time steps of the network. Thus, this article aims to reduce the memory cost of the SNN training based on reversible architectures.

\subsection{Reversible Architectures}

Reversible architectures are neural networks based on NICE reversible transformation \cite{dinh2014nice}. Reversible ResNet \cite{gomez2017reversible} is the first work that utilizes it for CNN-based image classification tasks. They employ reversible blocks to complete memory-efficient network training. The core of its memory saving is that the intermediate activation can be reconstructed through the reverse process. After that, other works \cite{hascoet2019layer, sander2021momentum, li2021m} have further iterated on the CNN-based reversible architectures. Recently, \cite{mangalam2022reversible} applied the reversible transformation to vision transformers and proposed Rev-ViT and Rev-MViT, two memory-efficient transformer structures. They found that reversible architectures have stronger inherent regularization than their non-reversible counterparts. In addition, reversible transformation has also been adopted in other networks, such as UNet \cite{brugger2019partially}, masked convolutional networks \cite{song2019mintnet}, and graph neural networks \cite{li2021training}.

It is worth noting that the above reversible architectures are reversible in the spatial dimension, in which the forward process propagates from shallow to deep layers, and the reverse process propagates from deep to shallow layers. Unlike them, reversible RNN \cite{revrnn} is reversible in the temporal dimension. It calculates hidden states in the past by reversing them from the future. SNN is a network with both spatial and temporal dimensions, while our spiking reversible block is reversible along the spatial dimension and consistent along the temporal dimension.

\section{Approach}
\begin{figure*}[htp]
\centering
\includegraphics[width=0.80\textwidth]{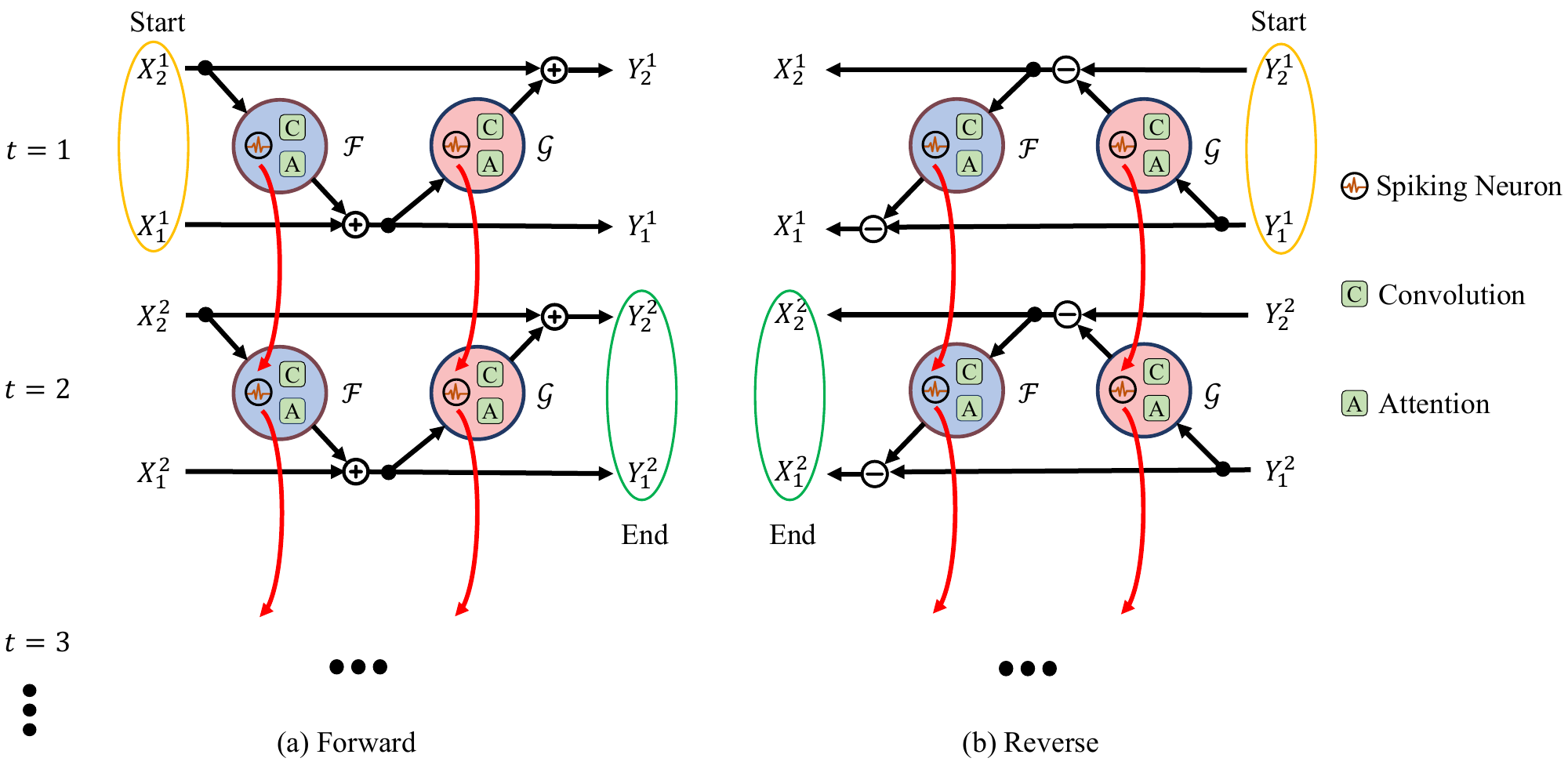} 
\caption{\small{Illustration of the forward (a) and reverse (b) process of our spiking reversible block. We can reconstruct the computational graph and recompute all intermediate variables in forward pass with the reverse process. Note that there is a reset process between them.}
}
\label{RevSNN}
\end{figure*}
In this section, we first explain the spiking neuron model, which is the preliminary of SNNs. Then, we present our proposed spiking reversible block. Furthermore, we apply it to spiking ResNet-like and transformer-like structures and propose the reversible spiking ResNet and reversible spiking transformer. They both support memory-efficient end-to-end training.

\subsection{Spiking neuron model}
\label{spikingneuron}
The spiking neuron, which plays the role of activation function, is the fundamental unit used in SNNs. It converts analog membrane potentials to binary spiking signals. The leaky-integrate-and-fire (LIF) neuron is a widely used spiking neuron whose discrete-time dynamics can be formulated as follows:
\begin{equation}
\label{LIF_model}
H[t]=V[t-1]+\frac{1}{\tau_{m}}\left(I[t]-\left(V[t-1]-V_{reset}\right)\right)
\end{equation}

\begin{equation}
\label{firing}
S[t] =\Theta\left(H[t]-V_{th}\right)
\end{equation}

\begin{equation}
\label{reset}
V[t] =H[t](1-S[t])+V_{reset}S[t]
\end{equation}
where $V[t]$ represents the membrane potential at time $t$, and $H[t]$ is the hidden membrane potential before trigger time $t$. $I[t]$ is the synaptic current, which is the input from other neurons. Once $H[t]$ exceeds the firing threshold $V_{th}$, the neuron will file a spike expressed by $S[t]$. Then, the membrane potential $V[t]$ will be reset to reset potential $V_{rest}$.

In addition to LIF, we also use (integrate-and-fire) IF neuron in this work, which is a simplified version of LIF. Its integrate dynamics (Eq.\ref{IF_model}) differs from LIF, while the fire and reset processes remain unchanged.

\begin{equation}
\label{IF_model}
H[t]=V[t-1]+ I[t]
\end{equation}

\subsection{Spiking reversible block}
\subsubsection{Computation graph of spiking reversible block}

During standard backpropagation training, a single-batch is computed with a forward-backward process. In contrast, for a reversible block, this computation turns to a forward-reverse-backward process. The added reverse process utilizes the output of the block to compute the input in reverse. Then we can delete all inputs and intermediate variables after the forward process and save only the output. RevNet \cite{gomez2017reversible} and RevRNN \cite{revrnn} implement the reversible blocks in the spatial and temporal dimensions, respectively.

For SNNs, as long as the network is designed in a two-residual-stream manner in \cite{gomez2017reversible}, we can establish the reverse process in the spatial dimension. However, in the temporal dimension, the reverse means that the input potential of all neurons must be calculated through their output spikes, which is theoretically impossible for spiking neurons described in Eq.~\ref{LIF_model}. Therefore, spiking reversible block should be reversible along the spatial dimension and consistent along the temporal dimension. We extend the single-batch computation process to forward-reset-reverse-backward. The computation graphs for forward and reverse processes are shown in Figure~\ref{RevSNN}, where $\mathcal{F}$ and $\mathcal{G}$ can be set as arbitrary spiking modules composed of spiking neurons, convolutional layers, fully connected layers, attention mechanisms, etc. Since spiking neurons have different membrane potentials at different time steps, $\mathcal{F}$ and $\mathcal{G}$ vary with time. We use $\mathcal{F}^t$ and $\mathcal{G}^t$ to represent these two modules at the time step $t$.

In the forward process, the starting node of the graph lies in the input node at time step $1$, and the end node is the output at time $T$, where $T$ is the total time steps of the SNN. At each time step $t$, output $Y^t$ is calculated using formula~\ref{forward_p}, as the horizontal arrows in Figure~\ref{RevSNN}a. From time step $t$ to $t+1$, the edges of the computation graph are established through the inherited membrane potential of all spiking neurons in $\mathcal{F}$ and $\mathcal{G}$, as the red arrows illustrate in Figure~\ref{RevSNN}a.
\begin{equation}
\label{forward_p}
\begin{aligned}
& Y_1^t=X_1^t+\mathcal{F}^t\left(X_2^t\right) \\
& Y_2^t=X_2^t+\mathcal{G}^t\left(Y_1^t\right)
\end{aligned}
\end{equation}

Before the reverse process, all spiking neurons are reset by resetting membrane potential to the initial state, which is named the reset process. 

In the reverse process, the starting node of the graph lies in the output node at time step $1$, and the end node is the input at time step $T$. For each time step $t$, input $X^t$ is calculated using formula~\ref{reverse_p}, as the reversed horizontal arrows in Figure~\ref{RevSNN}b. From time step $t$ to $t+1$, same as forward process, the edges of the computation graph are established through the inherited membrane potential of all spiking neurons in $\mathcal{F}$ and $\mathcal{G}$, as the red arrows show in Figure~\ref{RevSNN}b.

\begin{equation}
\label{reverse_p}
\begin{aligned}
& X_2^t=Y_2^t-\mathcal{G}^t\left(Y_1^t\right) \\
& X_1^t=Y_1^t-\mathcal{F}^t\left(X_2^t\right)
\end{aligned}
\end{equation}

\subsubsection{Learning without caching intermediate variables}
During network training, the backward process is essential for updating the network weights. Consider the presynaptic weight $W_l$ of a spiking neuron in the $l_{th}$ layer. Its gradient is calculated as follows:

\begin{equation}
\frac{\partial L}{\partial W_l}=\sum_t\left(\frac{\partial L}{\partial S_l^t} \frac{\partial S_l^t}{\partial U_l^t}+\frac{\partial L}{\partial U_l^{t+1}} \frac{\partial U_l^{t+1}}{\partial U_l^t}\right) \frac{\partial U_l^t}{\partial W_l}
\end{equation}

where $S_l^t$ and $ U_l^t$ are the output spike (activation) and membrane potential at time step $t$, which are calculated using the spiking neuron dynamics. It can be found that the gradient calculation requires all output spikes and membrane potentials at all time steps. In fact, almost all intermediate variables in the forward process are needed in the backward process. In standard training, these variables are cached in GPU memory after the forward process. Because of the sequential nature of the network, all intermediate variables for all layers at all time steps should be stored. Thus, peak memory usage becomes linearly dependent on the network depth $D$ and time steps $T$. Its spatial complexity is $O(D \cdot T)$.

For the training of the spiking reversible block, we propose Theorem~\ref{the1_app}, which means all intermediate variables in the forward process can be recomputed from output in the reverse process. Then, only output $Y$ needs caching in the forward process. Furthermore, if spiking reversible blocks are sequentially placed, we only need to store the output of the last block. Before the backward process of any block, we can recompute all intermediate variables with the output. In this process, the peak memory usage is the memory required for a single block whose spatial complexity is $O(T)$. Since direct training SNNs often have relatively small $T$ (such as 4), the peak memory usage during training is much smaller.
\begin{theorem}
\label{the1}
Consider a spiking reversible block with $T$ time steps, if the forward and reverse functions are formulated as Eq.~\ref{forward_p} and Eq.~\ref{reverse_p}, and outputs of forward process are fed into the reverse process, then $X^t$, $Y^t$ and all intermediate variables (including the intermediate activations and membrane potentials) in $\mathcal{F}^t$ and $\mathcal{G}^t$ in the forward process are the identical to those in the reverse process.
\end{theorem}

\begin{proof}
The proof of Theorem~\ref{the1} is presented in the Appendix.
\end{proof}

\subsection{Reversible spiking residual neural network}

\begin{figure}[tp]
\centering
\includegraphics[width=0.8\columnwidth]{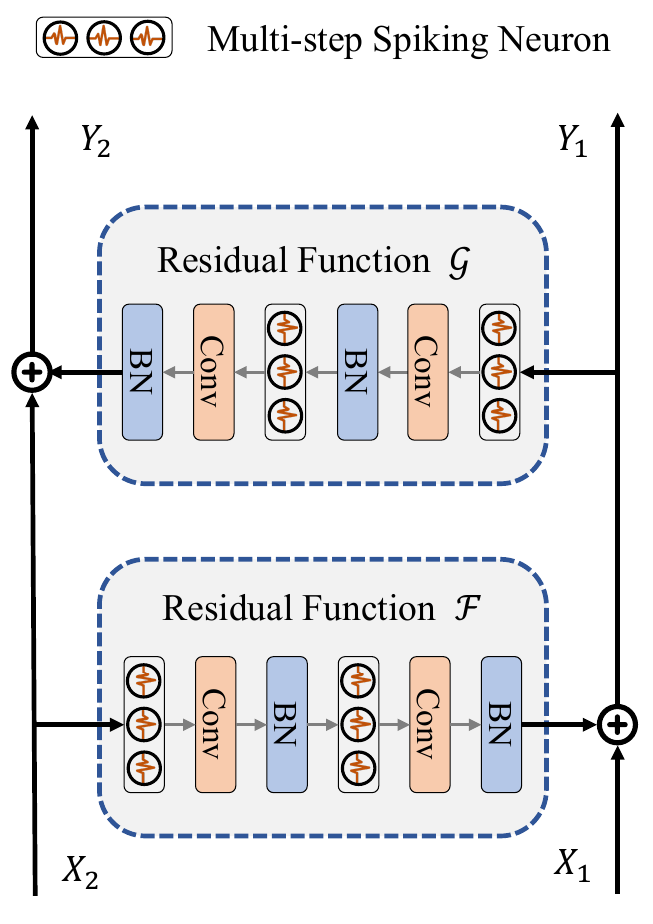} 
\caption{\small{Basic block of RevSResNet. We utilize two residual functions with the same structure as $\mathcal{F}$ and $\mathcal{G}$.}
}
\label{RevSResNet}
\end{figure}

ResNet \cite{he2016deep} is one of the most popular deep convolutional neural networks (CNNs), and residual learning is also the best solution for CNN-based SNNs to tackle the gradient degradation problem \cite{sew}. With the help of our spiking reversible block, we propose the reversible spiking residual neural network, which completes the training of deep SNNs with much less memory usage.
\subsubsection{Basic block}
In ANN ResNet, the parameterized residual function is wrapped around a single residual stream in each block. We adopt it to the spiking reversible block and propose the two-residual-stream architecture in Figure~\ref{RevSResNet}. The input $X$ is partitioned into tensors $X_1$ and $X_2$ in halves along the channel dimension. The forward process follows transformation in Eq.~\ref{forward_p} to ensure reversibility. We utilize two residual functions with the same structure as $\mathcal{F}$ and $\mathcal{G}$. To ensure that all operations are spike computations, we adopt the Activation-Conv-BatchNorm paradigm \cite{hu2021msresnet}. Each residual function consists of two sequentially connected multi-step spiking neurons, convolutional layers, and batch normalization.

\subsubsection{Downsample block}
Due to the reversibility of the basic block, the feature dimensions of $X$ and $Y$ are identical. Therefore, residual functions $\mathcal{F}$ and $\mathcal{G}$ must be equidimensional in input and output spaces, which means that downsample layers (such as maxpooling or convolution with a stride of 2) cannot appear in spiking reversible blocks. To replace the downsampling basic blocks in ResNet, we set up a downsample block at the start of the stages where downsampling is required. We first use a $3 \times 3$ average pooling with a stride of 2 to downsample the image scale and then increase the feature channels using a $1\times 1$ convolutional layer with a stride of 1.

\subsubsection{Network architecture}
The high-level structure of RevSResNet is the same as its non-reversible counterpart MS ResNet \cite{hu2021msresnet}. The first convolution is regarded as the encoding layer which performs the initial downsampling. Then the spiking features propagate through the four stages with basic blocks. We set up a downsample block at the start of the second to fourth stages. The network ends with an average pooling and fully connected layer. 

When spiking reversible blocks are sequentially connected (we call it reversible sequence), we only need to store the output of the last block to complete the training. Leave out the downsample block, all stages in RevSResNet are reversible sequences. No matter how the number of blocks in a reversible sequence grows, the memory usage required by intermediate variables does not increase. The detailed architectures of RevSResNet are summarized in Table~\ref{details}. RevSResNet-$N$ means the network with $N$ layers.

\begin{table}[]
\centering
\begin{tabular}{c|c}
\hline \hline
      Total layers                & $N = 5 + 4 * \sum n_i$                  \\ \hline
conv1                 & 3$\times$3, 128                  \\ \hline
reversible sequence 1 & \multicolumn{1}{c}{\begin{tabular}[c]{@{}c@{}}$\left(\begin{array}{l}3 \times 3,64 \\ 3 \times 3,64\end{array}\right) \times 2 \times n_1 $\end{tabular}}                        \\ \hline
reversible sequence 2 & \multicolumn{1}{c}{\begin{tabular}[c]{@{}c@{}}$\left(\begin{array}{l}3 \times 3,128 \\ 3 \times 3,128\end{array}\right)^{*} \times 2 \times n_2 $\end{tabular}}                        \\ \hline
reversible sequence 3 & \multicolumn{1}{c}{\begin{tabular}[c]{@{}c@{}}$\left(\begin{array}{l}3 \times 3,256 \\ 3 \times 3,256\end{array}\right)^{*} \times 2 \times n_3 $\end{tabular}}                        \\ \hline
reversible sequence 4 & \multicolumn{1}{c}{\begin{tabular}[c]{@{}c@{}}$\left(\begin{array}{l}3 \times 3,448 \\ 3 \times 3,448\end{array}\right)^{*} \times 2 \times n_4 $\end{tabular}}                        \\ \hline
                      & average pool, fc, softmax \\ \hline \hline
\end{tabular}
\caption{Architectures of RevSResNet. The stride of conv1 are set to 2 for downsampling. $*$ means that a downsample block is set at the beginning of the reversible sequence. $N$ represents the total number of layers.}
\label{details}
\end{table}

\subsection{Reversible spiking transformer}

\begin{figure}[htp]
\centering
\includegraphics[width=0.8\columnwidth]{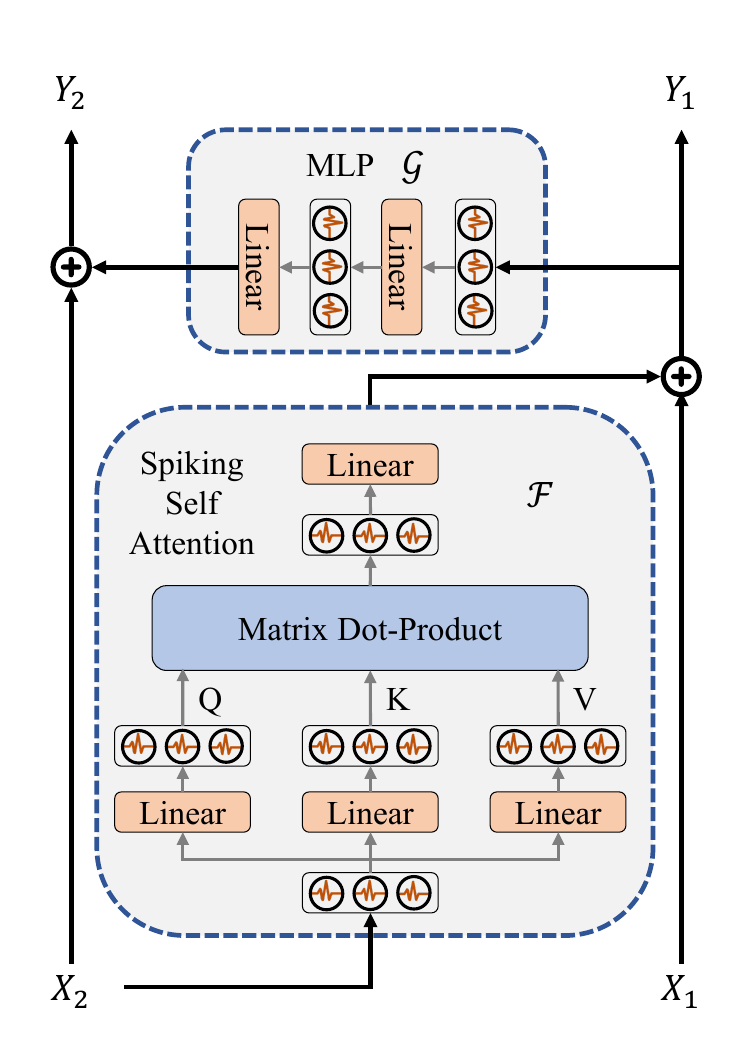} 
\caption{\small{Basic block of our RevSFormer. We consider spiking self-attention and spiking MLP block as $\mathcal{F}$ and $\mathcal{G}$}, respectively.
}
\label{RevSFormer}
\end{figure}

Vision transformer has taken the accuracy of computer vision tasks to a new level. Combining our spiking reversible block with the spiking transformer \cite{zhou2023spikingformer}, we propose RevSFormer and prove the feasibility of reversible structures in transformer-like SNNs.

\subsubsection{Basic block}
Unlike ResNet, a spiking transformer block has two relatively independent residual functions: spiking self-attention (SSA) and spiking MLP block (MLP). They are wrapped around their residual connection, respectively. Under this condition, we respectively consider SSA and MLP as $\mathcal{F}$ and $\mathcal{G}$, and propose the basic block in RevSFormer, as is shown in Figure~\ref{RevSFormer}. We adopt the same SSA and MLP structure as Spikingformer \cite{zhou2023spikingformer}, so our basic block's computational complexity and parameter numbers are consistent with the original spiking transformer block.

\subsubsection{Network structure}
The high-level structure of RevSFormer is the same as its non-reversible counterpart Spikingformer. The network includes a spiking tokenizer, $L$ basic blocks, and a classification head. The spiking tokenizer computes the patch embedding of the image and projects the embedding into a fixed size with several convolutional and maxpooling layers. The classification head is composed of a spiking neuron and a fully connected layer. It is worth mentioning that all downsampling operations of RevSFormer are placed in the spiking tokenizer. Since there are no other downsampling or irreversible operations between all basic blocks, RevSFormer has only one reversible sequence composed of $L$ basic blocks. As $L$ grows, the memory required to store intermediate variables is expected to stay the same. The detailed configurations of RevSFormer are the same as Spikingfomer. And RevSFormer-$L$-$D$ means the network has $L$ blocks and the embedding dimension is $D$.

\section{Experiments}
\begin{table*}[htp]
\centering

\scalebox{0.93}{
\begin{tabular}{lcccclcc}
\hline
Methods           & Architecture        & \begin{tabular}[c]{@{}c@{}}Param\\ (M)\end{tabular} & \begin{tabular}[c]{@{}c@{}}Time\\ Step\end{tabular} & \begin{tabular}[c]{@{}c@{}}FLOPS\\ (G)\end{tabular} & \begin{tabular}[c]{@{}c@{}}Memory\\ (MB/img)\end{tabular} & \begin{tabular}[c]{@{}c@{}}CIFAR10\\ Top-1 Acc\end{tabular} & \begin{tabular}[c]{@{}c@{}}CIFAR100\\ Top-1 Acc\end{tabular} \\ \hline
Hybrid training \cite{hybrid_training}   & VGG-11              & 9.27                                                & 125                                                 & -                                                   & -                                                         & 92.22                                                       & 67.87                                                       \\
Diet-SNN \cite{dietsnn}          & ResNet-20           & 0.27                                                & 10                                                  & -                                                   & -                                                         & 92.54                                                       & 64.07                                                       \\
STBP  \cite{stbp}            & CIFARNet            & 17.54                                               & 12                                                  & -                                                   & -                                                         & 89.83                                                       & -                                                           \\
STBP NeuNorm \cite{neuronorm}     & CIFARNet            & 17.54                                               & 12                                                  & -                                                   & -                                                         & 90.53                                                       & -                                                           \\
TSSL-BP  \cite{tssl}        & CIFARNet            & 17.54                                               & 5                                                   & -                                                   & -                                                         & 91.41                                                       & -                                                           \\
STBP-tdBN  \cite{tdbn}       & ResNet-19           & 12.63                                               & 4                                                   & -                                                   & -                                                         & 92.92                                                       & 70.86                                                       \\
TET  \cite{tet}             & ResNet-19           & 12.63                                               & 4                                                   & -                                                   & -                                                         & 94.44                                                       & 74.47                                                       \\
DS-ResNet \cite{dsresnet}        & ResNet20    & 4.32                                                & 4                                                   & -                                                   & -                                                         & 94.25                                                       & -                                                       \\ 
Spikformer \cite{zhou2022spikformer}        & Spikformer-4-384    & 9.32                                                & 4                                                   & -                                                   & -                                                         & 95.19                                                       & 77.86                                                       \\ \hline
MS ResNet \cite{hu2021msresnet}         & MS ResNet18         & 11.22                                               & 4                                                   & 2.22                                                & 54.83                                                     & 94.40                                                       & 75.06                                                       \\
 
RevSResNet (ours) & RevSResNet21        & 11.05                                               & 4                                                   & 2.38                                                & \textbf{23.59 \small$\downarrow 2.32 \times$}                                            & 94.53                                                       & 75.46                                                       \\
MS ResNet \cite{hu2021msresnet}         & MS ResNet34         & 21.33                                             & 4                                                   & 4.64                                                & 89.33                                                     & 94.69                                                       & 75.34                                                       \\
 
RevSResNet (ours) & RevSResNet37        & 23.59                                               & 4                                                   & 4.66                                                & \textbf{23.58 \small$\downarrow 3.79 \times$}                                            & 94.77                                                       & 76.34                                                       \\ \hline
Spikingformer \cite{zhou2023spikingformer}     & Spikingformer-2-384 & 5.76                                                & 4                                                   & 2.79                                                & 83.05                                                     & 95.12                                                       & 77.96                                                       \\

RevSFormer (ours) & RevSFormer-2-384    & 5.76                                                & 4                                                   & 2.79                                                & \textbf{41.68 \small$\downarrow 1.99 \times$}                                            & 95.29                                                       & 78.04                                                       \\
Spikingformer \cite{zhou2023spikingformer}     & Spikingformer-4-384 & 9.32                                                & 4                                                   & 3.70                                                & 125.06                                                    & 95.35                                                       & 79.02                                                       \\
 
RevSFormer (ours) & RevSFormer-4-384    & 9.32                                                & 4                                                   & 3.70                                                & \textbf{41.74 \small$\downarrow 3.00 \times$}                                            & 95.34                                                       & 79.04                                                       \\ \hline
\end{tabular}}
\caption{Comparison to prior works on static datasets, CIFAR-100 and CIFAR10. Note that results of MS ResNet and Spikingformer are based on our implementation for a fair comparison. Bold values denotes the memory usage of our reversible SNNs.}
\label{CIFAR10_100}
\end{table*}

We evaluate the performance of our reversible structures on static datasets (CIFAR10 and CIFAR100) and neuromorphic datasets (CIFAR10-DVS and DVS128 Gesture). The metrics include parameters, time steps, FLOPS, memory per image, and the top-1 accuracy. The memory per image is measured as the peak GPU memory each image occupies during training. To ensure direct comparability with non-reversible counterparts, we match the model complexity (FLOPS in metric) and number of parameters as closely as possible. The dataset introduction, detailed network configuration, and other experimental settings are presented in the Appendix.

\subsection{Experiment on static datasets}
CIFAR10 and CIFAR100 each provides 50000 train and 10000 test images. On these datasets, we establish two comparisons (MS ResNet18 vs. RevSResNet21, MS ResNet34 vs. RevSResNet37) for ResNet-like structures. For transformer-like structures, the network configuration and model complexity of RevSFormer are identical to Spikingformer. Results are shown in Table~\ref{CIFAR10_100}.

From an accuracy perspective, we find that the performance of RevSResNet and RevSFormer is comparable to their counterparts with similar complexity. RevSResNet37 achieves 94.77\% and 76.34\% accuracy on CIFAR10 and CIFAR100 datasets, respectively, while RevSFormer-4-384 achieves 95.34\% and 79.04\% accuracy with a time step of 4. The performance of RevSResNet and RevSFormer is even slightly better than MS ResNet and SpikingFormer, which may be due to stronger inherent regularization of reversible architectures than vanilla networks \cite{mangalam2022reversible}.

From the memory perspective, our reversible SNNs are much more memory-efficient than vanilla SNNs. On one hand, RevSResNet37 and RevSFormer-4-384 consume 23.58 and 41.74 MB GPU memory per image, which is $3.79 \times$ and $3.00\times$ lower than their counterparts. On the other hand, the memory usage does not increase with depth in our networks, which will be further discussed later.
\begin{table*}[tp]
\centering

\scalebox{0.94}{
\begin{tabular}{lclcccc}
\hline
                                            &                                                                       &                                                                             & \multicolumn{2}{c}{CIFAR10-DVS}          & \multicolumn{2}{c}{DVS128 Gesture}                            \\ \cline{4-7} 
\multirow{-2}{*}{Methods}                   & \multirow{-2}{*}{\begin{tabular}[c]{@{}c@{}}FLOPS\\ (G)\end{tabular}} & \multirow{-2}{*}{\begin{tabular}[c]{@{}c@{}}Memory\\ (MB/img)\end{tabular}} & Time Step & Top-1 Acc                    & Time Step                  & Top-1 Acc                     \\ \hline
LIAF-Net \cite{wu2021liaf}                                    & -                                                                     & -                                                                           & 10        & 70.40                         & 60                         & 97.56                          \\
TA-SNN \cite{TAsnn}                                      & -                                                                     & -                                                                           & 10        & 72.00                           & 60                         & 98.61                          \\
Rollout  \cite{Rollout}                                   & -                                                                     & -                                                                           & 48        & 66.75                         & 240                        & 97.16                          \\
tdBN \cite{tdbn}                                       & -                                                                     & -                                                                           & 10        & 67.80                         & 40                         & 96.87                          \\
PLIF \cite{plif}                                       & -                                                                     & -                                                                           & 20        & 74.80                         & 20                         & 97.57                          \\
SEW ResNet \cite{sew}                                 & -                                                                     & -                                                                           & 16        & 74.40                         & 16                         & 97.92                          \\
Dspike \cite{Dspike}                                     & -                                                                     & -                                                                           & 10        & 75.40                         & -                          & -                             \\
DSR \cite{dsr}                                        & -                                                                     & -                                                                           & 10        & 77.27                         & -                          & -                             \\
DS-ResNet \cite{dsresnet}                                 & -                                                                     & -                                                                           & 10        & 70.36                         & 40                         & 97.29                          \\
Spikformer \cite{zhou2022spikformer}                                 & -                                                                     & -                                                                           & 16        & 80.60                         & 16                         & 97.90                          \\ \hline
MS ResNet20     \cite{hu2021msresnet}                             & 0.42                                                                  & 50.72                                                                       & 10        & 76.00                         & 10                         & 94.79                         \\
 
RevSResNet24 (ours) & 0.43                                                                  & \textbf{24.97 \small$\downarrow 2.03 \times$}                                                              & 10        & 75.50 & 10                         & 94.44 \\
MS ResNet20 \cite{hu2021msresnet}                                & 0.67                                                                  & 79.38                                                                       & 16        & 75.80                         & 16                         & 97.57                         \\

RevSResNet24 (ours)                         & 0.69                                                                  & \textbf{39.52 \small$\downarrow 2.01 \times$}                                                              & 16        & 76.40                         & 16                         & 96.53                         \\ \hline
Spikingformer-2-256    \cite{zhou2023spikingformer}                     & 3.78                                                                  & 295.73                                                                      & 10        & 78.50                         & 10                         & 96.88                         \\

RevSFormer-2-256 (ours)                     & 3.78                                          & \textbf{227.50 \small$\downarrow 1.30 \times$}                                     & 10        & 81.40                         & 10                         & 97.22                         \\
Spikingformer-2-256   \cite{zhou2023spikingformer}                        & 6.05                                                                  & 466.08                                                                      & 16        & 80.30  & 16        & 98.26                         \\
 
RevSFormer-2-256 (ours)                     & 6.05                                          & \textbf{359.58 \small$\downarrow 1.30 \times$}                                     & 16        & 82.20                         & 16                         & 97.57                         \\ \hline
\end{tabular}}
\caption{Comparisons with prior works on neuromorphic datasets, CIFAR10-DVS and DVS128 Gesture. Note that results of MS ResNet and Spikingformer are based on our implementation for a fair comparison. Bold values denote the memory usage of our reversible SNNs.}
\label{dvsdataset}
\end{table*}

\subsection{Experiment on neuromorphic datasets}
On the neuromorphic datasets, we conduct experiments with two different time steps, 10 and 16. And we establish one network comparison (MS ResNet20 vs. RevSResNet24) for ResNet-like structures. For transformer-like structures, the network configuration are identical between reversible and non-reversible structures. 

Results are shown in Table~\ref{dvsdataset}. The relative changes in accuracy and memory are similar to those on static datasets. Our RevSResNet and RevSFormer achieve a memory usage reduction of $2.01 \times$ and $1.30 \times$, respectively. And the magnitude of the reduction stays consistent across different time steps. In terms of performance, RevSResNet24 and RevSFormer-2-256 achieve 76.4\% and 82.2\% accuracy on CIFAR10-DVS dataset with a time step of 16.

\subsection{Ablation studies}
\subsubsection{Memory usage vs. depth}
Theoretically, for a reversible sequence, the memory usage required by intermediate variables does not increase with the number of reversible blocks because we only need to save the output of the whole sequence. Thus, for RevSResNet with 4 reversible sequences and RevSFormer with 1 sequence, the memory usage per image should not increase with depth. Figure~\ref{gpu_cost} plots the memory usage for our reversible SNNs and their counterparts. For ResNet-like structures, the relative memory saving magnitude increases up to $8.1\times$ as the model goes deeper. For transformer networks, our RevSFormer-16-384 saves $9.1\times$ GPU memory per image. It is expected that this memory saving magnitude will increase further with increasing depth.

\subsubsection{Memory usage vs. time step}
The memory required by an SNN is $T$ times larger than an ANN. Thus, the GPU memory required per image grows linearly with the total time steps $T$. Figure~\ref{mem_t} shows the relationship between memory usage and time steps. As is seen, for each model, the memory usage increases with a certain slope $m$. In our reversible SNNs, intermediate variables in the non-reversible parts (e.g., the downsample layers and the spiking tokenizer) and the output of each reversible sequence still need caching. Thus, memory usage is not decoupled from time steps $T$. However, through reversible architecture, we have greatly reduced the slope of memory usage growth from 28.5 and 20.2 of non-reversible SNNs to 9.6 and 5.3 of our reversible networks.

\begin{figure}[htp]
\centering
\includegraphics[width=0.95\columnwidth]{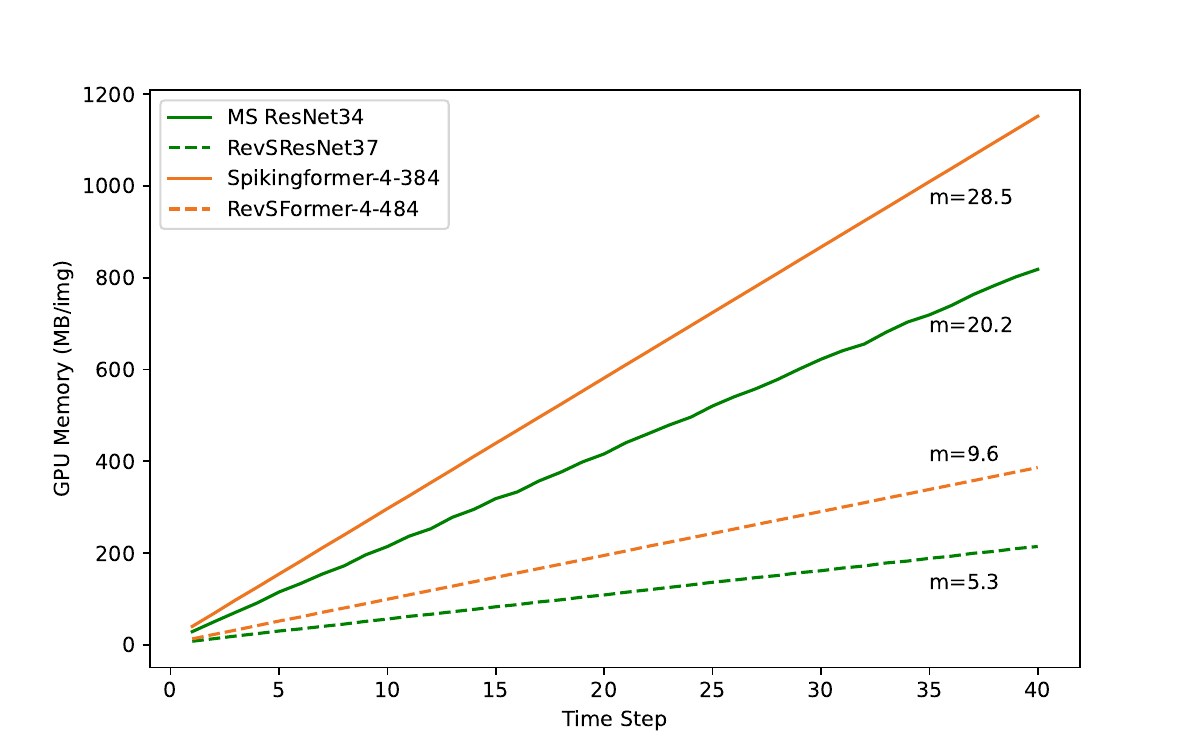} 
\caption{\small{The relationship between GPU memory and time step.}
}
\label{mem_t}
\end{figure}

\subsubsection{Computational overhead during training}

In general, for a network with $N$ operations, the forward and backward processes take $N$ and $2N$ operations approximately \cite{gomez2017reversible}. Our spiking reversible block requires the extra reset and reverse processes. The reset process take negligible operations and the reverse process take $N$ operations, same as forward. In summary, the reversible architectures need roughly 33\% more computations than vanilla networks during training. Besides, reversible SNNs have larger maximum batch size, which may slightly influence the training speed and final performance \cite{wu2021rethinking}.

The training time and maximum batch size of our reversible SNNs and their counterparts are shown in Table~\ref{trainingtime}. The values are measured on a single 24GB RTX3090 GPU under CIFAR10 dataset. Our RevSResNet37 takes $1.33\times$ more training time in practice. Besides, it achieves a $2.69\times$ increase in maximum batch size, and the increase magnitude will go larger on bigger models.

\begin{table}[htp]

\begin{tabular}{lll}
\hline
Architecture        & \begin{tabular}[c]{@{}l@{}}Training time\\ (seconds / epoch)\end{tabular} & \begin{tabular}[c]{@{}l@{}}Maximum\\ Batch size\end{tabular} \\ \hline
MS ResNet34         & 98                                                                      & 239                                                            \\
RevSResNet37        & \textbf{131 \small$\uparrow 1.33 \times$}                                                             & \textbf{644 \small$\uparrow 2.69 \times$}                                                  \\ \hline
Spikingformer-4-384 & 105                                                                      & 164                                                            \\
RevSFormer-4-384    & \textbf{133 \small$\uparrow 1.27 \times$}                                                             & \textbf{286 \small$\uparrow 1.74 \times$}                                                   \\ \hline
\end{tabular}
\caption{The training time and maximum batch size of our reversible structures and their non-reversible counterparts.}
\label{trainingtime}
\end{table}
\section{Conclusion}
In this paper, we propose the reversible spiking neural network to reduce the memory cost of intermediate activations and membrane potentials during training of SNNs. We first extend the reversible architecture along temporal dimension and propose the reversible spiking block, which can reconstruct the computational graph of forward pass with a reverse process. On this basis, we present the RevSResNet and RevSFormer models, which are the reversible counterparts of the state-of-the-art SNNs. Through experiments on static and neuromorphic datasets, we demonstrate that the memory cost per image of our reversible SNNs does not increase with the network depth. In addition, RevSResNet and RevSFormer achieve comparative accuracies and consume much less GPU memory than their counterparts with roughly identical model complexity and parameters.

\section{Acknowledgments}
This work was supported by STI 2030-Major Projects 2021ZD0201403, in part by NSFC 62088101 Autonomous Intelligent Unmanned Systems.

\bibliography{main}

\appendixpage
\setcounter{theorem}{0}
\renewcommand\thesection{\Alph{section}}
\begin{table*}[tp]
\centering
\begin{tabular}{cccccc}
\hline \hline
Dataset                       & Network      & Blocks  & Channels       & Params(M) & FLOPS(G) \\ \hline
\multirow{2}{*}{static}       & MS ResNet18  & 2-2-2-2 & 64-128-256-512 & 11.22     & 2.22     \\
                              & RevSResNet21 & 1-1-1-1 & 64-128-256-448 & 11.05     & 2.38     \\ \hline
\multirow{2}{*}{static}       & MS ResNet34  & 3-4-6-3 & 64-128-256-512 & 21.33     & 4.64     \\
                              & RevSResNet37 & 1-2-3-2 & 64-128-256-448 & 23.59     & 4.66     \\ \hline
\multirow{2}{*}{neuromorphic} & MS ResNet20 & 3-3-3   & 16-32-64       & 0.27      & 0.42     \\
                              & RevSResNet24 & 1-2-2   & 16-32-48       & 0.26      & 0.43     \\ \hline
                              \hline
\end{tabular}
\caption{Detailed configurations of ResNet-like SNNs.}
\label{network}
\end{table*}
\section{Proof of the Theorem}
\begin{theorem}
\label{the1_app}
Consider a spiking reversible block with $T$ time steps, if the forward and reverse functions are formulated as Eq.~\ref{forward_p_app} and Eq.~\ref{reverse_p_app}, and outputs of forward process are feed into the reverse process, then $X^t$, $Y^t$ and all intermediate variables (including the intermediate activations and membrane potentials) in $\mathcal{F}^t$ and $\mathcal{G}^t$ in the forward process are the identical to those in the reverse process.
\end{theorem}
\begin{equation}
\label{forward_p_app}
\begin{aligned}
& Y_{f1}^t=X_{f1}^t+\mathcal{F}_f^t\left(X_{f2}^t\right) \\
& Y_{f2}^t=X_{f2}^t+\mathcal{G}_f^t\left(Y_{f1}^t\right)
\end{aligned}
\end{equation}

\begin{equation}
\label{reverse_p_app}
\begin{aligned}
& X_{r2}^t=Y_{r2}^t-\mathcal{G}_r^t\left(Y_{r1}^t\right) \\
& X_{r1}^t=Y_{r1}^t-\mathcal{F}_r^t\left(X_{r2}^t\right)
\end{aligned}
\end{equation}

\begin{proof}
To facilitate the distinction, we use $X_f^t$, $Y_f^t$, and $\mathcal{F}_f^t$, $\mathcal{G}_f^t$ to represent the input, output and spiking modules of forward process at time step $t$. And $Xr^t$, $Yr^t$, and $\mathcal{F}_r^t$, $\mathcal{G}_r^t$ represent the input, output and spiking modules of reverse process at time step $t$. Then the purpose is to prove that for each time step $t$, $X_f^t = X_r^t$,  $\mathcal{F}_f^t = \mathcal{F}_r^t$, and $\mathcal{G}_f^t = \mathcal{G}_r^t$ in the condition of $Y_f^t=Y_r^t$.

As spiking modules in different processes share the same parameter, the main difference lies in the membrane potentials. Let $V_f^t$ and $V_r^t$ be the membrane potentials respectively.

At time step $1$, as all membrane potentials are initialized to zero, then:
\begin{equation}
    V_f^1 = V_r^1 = 0
\end{equation}
So $\mathcal{F}_f^1 = \mathcal{F}_r^1$, and $\mathcal{G}_f^1 = \mathcal{G}_r^1$ at time step $1$.

Because $Y_{f1}^1=Y_{r1}^1$, which are the inputs of $\mathcal{G}_f^1 $ and $\mathcal{G}_r^1$, then all intermediate variables and outputs of $\mathcal{G}_f^1 $ and $\mathcal{G}_r^1$ are the same. That is:
\begin{equation}
    \mathcal{G}_f^1\left(Y_{f1}^1\right) = \mathcal{G}r^1\left(Y_{r1}^1\right)
\end{equation}

According to Eq.~\ref{forward_p_app} and~\ref{reverse_p_app}, we can prove that $X_{f2}^1=X_{r2}^1$, which are the inputs of $\mathcal{F}_f^1 $ and $\mathcal{F}_r^1$. Then all intermediate variables and outputs of $\mathcal{F}_f^1 $ and $\mathcal{F}_r^1$ are the same. That is 
\begin{equation}
    \mathcal{F}_f^1\left(X_{f2}^1\right) = \mathcal{F}_r^1\left(X_{r2}^1\right)
\end{equation}

Then we can prove that $X_{f1}^1=X_{r1}^1$.

So far, we have proven that $X_f^1 = X_r^1$,  $\mathcal{F}_f^1 = \mathcal{F}_r^1$, and $\mathcal{G}_f^1 = \mathcal{G}_r^1$. And all updates of membrane potentials in forward and reverse process are the same, which means:
\begin{equation}
\label{eq6}
    V_f^2 = V_r^2
\end{equation}

From the above proof, we can find that the equality of membrane potentials is the sufficient condition for the equality of other variables.
Similarly, according to Eq.~\ref{eq6}, we can also get $X_f^2 = X_r^2$,  $\mathcal{F}_f^2 = \mathcal{F}_r^2$, $\mathcal{G}_f^2 = \mathcal{G}_r^2$, and $V_f^3 = V_r^3$ at time step $2$.

By further reasoning, for each time step $t$, $X_f^t = X_r^t$, $\mathcal{F}_f^t = \mathcal{F}_r^t$, and $\mathcal{G}_f^t = \mathcal{G}_r^t$. And all intermediate variables in $\mathcal{F}^t$ and $\mathcal{G}^t$ of the forward and reverse process are the identical to each other.
\end{proof}

\section{Implementation details}

\subsection{Dataset introduction and pre-processing}
\subsubsection{CIFAR100 and CIFAR10}

CIFAR100 and CIFAR10 \cite{cifar} are two static image classification datasets that contain 60000 images each with a resolution of 32 × 32 pixels. CIFAR100 has 100 classes with 500 training images and 100 testing images per class, while CIFAR10 has 10 classes with 5000 training images and 1000 testing images per class. In the experiments of ResNet-like SNNs, we preprocess the images by applying random cropping with a size of 32 and a padding of 4, and horizontal flipping for data augmentation. We also normalize the images by subtracting the mean pixel intensity and dividing by the standard deviation, so that the images have zero mean and unit variance. While for transformer-like SNNs, we adopt a stronger augmentation following \cite{zhou2023spikingformer}, which includes random augmentation, mixup, and cutmix.

\subsubsection{CIFAR10-DVS}
CIFAR10-DVS \cite{li2017dvscifar10} is a dataset that consists of 10000 images in the form of spike trains. It is derived from recording the motion of CIFAR10 images on a LCD monitor using a DVS camera. We use the AER data pre-processing \cite{plif} to split each event into 10 or 16 slices (corresponding to time steps). The data augmentations are adopted from Spikingformer \cite{zhou2023spikingformer}.

\begin{table*}[tp]
\centering
\begin{tabular}{ccccccc}
\hline
Dataset                       & Structure              & Neuron & Epoch & Learning Rate & Optimizer & Batch Size \\ \hline
\multirow{2}{*}{CIFAR10(100)} & ResNet                 & IF     & 200   & 0.1           & SGD       & 32         \\
                              & Transformer            & LIF    & 310   & 0.0005        & AdamW     & 64         \\ \hline
CIFAR10-DVS                   & ResNet and Transformer & LIF    & 106   & 0.001         & AdamW     & 16         \\
DVS128 Gesture                & ResNet and Transformer & LIF    & 202   & 0.001         & AdamW     & 16         \\ \hline
\end{tabular}
\caption{Training settings of different structures on all datasets.}
\label{tb2}
\end{table*}

\begin{table*}[htp]
\centering
\scalebox{0.85}{
\begin{tabular}{cclllll}
\hline
Metrics & Architecture & RTX3080 (10G) & RTX3090 (24G) & RTX 4090 (24G) & Tesla V100 (32G) & Tesla A100 (40G) \\ \hline
\multirow{2}{*}{\begin{tabular}[c]{@{}c@{}}Maximum \\ Batch Size\end{tabular}} & MS ResNet34 & 74 & 239 & 237 & 337 & 417 \\
& RevSResNet37 & 214 \small$\uparrow2.89\times$ & 644 \small$\uparrow2.69\times$ & 644 \small$\uparrow2.72\times$ & 944 \small$\uparrow2.80\times$ & 1127 \small$\uparrow2.70\times$ \\ \hline
\multirow{2}{*}{\begin{tabular}[c]{@{}c@{}}Maximum \\ Time Step\end{tabular}} & MS ResNet34 & 4 & 16 & 15 & 22 & 28 \\
 & RevSResNet37 & 13 \small$\uparrow3.25\times$ & 46 \small$\uparrow2.88\times$ & 46 \small$\uparrow3.07\times$ & 66 \small$\uparrow3.00\times$ & 78 \small$\uparrow2.79\times$ \\ \hline
\multirow{2}{*}{\begin{tabular}[c]{@{}c@{}}Allocated Memory\\ (GB)\end{tabular} } & MS ResNet34 & 8.2 & 8.2 & 8.3 & 7.6 & 8.4 \\
 & RevSResNet37 & 4.5\small$\downarrow1.81\times$ & 4.5\small$\downarrow1.81\times$  & 4.6\small$\downarrow1.79\times$  & 4.2\small$\downarrow1.81\times$  & 4.7\small$\downarrow1.79\times$  \\ \hline
\multirow{2}{*}{\begin{tabular}[c]{@{}c@{}}Training Time\\ (seconds / epoch)\end{tabular}} & MS ResNet34 & 120.6 & 98 & 60.6 & 183.8 & 54.8 \\
 & RevSResNet37 & 157.4 \small$\uparrow1.31\times$ & 130.8 \small$\uparrow1.33\times$ & 81.4 \small$\uparrow1.34\times$ & 242.6 \small$\uparrow1.32\times$ & 78 \small$\uparrow1.42\times$ \\ \hline
\multirow{2}{*}{GPU-Utilization} & MS ResNet34 & 99\% & 99\% & 98\% & 99\% & 88\% \\
 & RevSResNet37 & 100\% & 99\% & 95\% & 99\% & 82\% \\ \hline
\end{tabular}}

\caption{Analysis on different GPU models. The last three rows are measured under batch size 64 and time step 4.}
\label{analysis}
\end{table*}

\subsubsection{DVS128 Gesture}
DVS128 Gesture \cite{dvs_gesture} is a neuromorphic dataset that consists of 11 categories of hand gestures from 29 people in three lighting conditions. The resolution of DVS128 Gesture is $128\times 128$. We use the same main AER data pre-processing and augmentations as for CIFAR10-DVS classification.

\subsection{Detailed architectures of reversible SNNs}
\subsubsection{ResNet-like SNNs}
The high-level structure of RevSResNet is the same as its non-reversible counterpart MS ResNet \cite{hu2021msresnet}. To ensure direct comparability with non-reversible counterparts, we match the model complexity (FLOPS in metric) and number of parameters as closely as possible. On static CIAFR datasets, we establish two comparisons (MS ResNet18 vs. RevSResNet21, MS ResNet34 vs. RevSResNet37). On neuromorphic CIFAR10-DVS and DVS128 Gesture datasets, we adopt the three-stage ResNet configuration and establish one comparison (MS ResNet 20 and RevSResNet24). The differences in configuration lie in the number of blocks and channel numbers in each stage, which are shown in Table~\ref{network}.


\subsubsection{Transformer-like SNNs}
For the transformer-like SNNs, due to our special design of basic reversible block, the only difference of RevSFromer and Spikingformer \cite{zhou2023spikingformer} lies in the type of basic block. Thus, the detailed configurations of RevSResNet follows Spikingformer.

\subsection{Training settings}
In our experiments, the implementation and GPU acceleration of all neurons are based on the PyTorch and SpikingJelly \cite{SpikingJelly} frameworks. SNNs for CIFAR10 and CIFAR100 datasets adopt a time step of 4. And SNNs for CIFAR10-DVS and DVS128 Gesture dataset adopts time steps of 10 and 16 for better comparison. The optimizer, spiking neuron and other training settings are listed in Table~\ref{tb2}. The code is available at https://github.com/mi804/RevSNN.git.

\section{Analysis on memory efficiency}
\subsection{Memory usage vs. embedding dimensions}
In the section Experiment, we have analyzed the relationship between GPU memory usage and network depth or time steps. In fact, memory usage is also related to the network width (channel number), which is reflected as the embedding dimension in the transformer-like SNNs. Here we explore the relationship between memory usage and embedding dimension on RevSformer and Spikingformer structures, as is shown in Figure~\ref{m_vs_dim}. The number of blocks for all networks is set to 4. As can be seen, for each model, memory usage increases linearly with a slope $m$. In our reversible SNNs, intermediate variables in the non-reversible parts and the output of each reversible sequence still need caching. However, through reversible architecture, we have greatly reduced the slope of memory usage growth from 0.32 of Spikingformer to 0.11 of our RevSFormer. For the largest embedding dimension of 1536, RevSFormer saves 2.94 $\times$ memory usage.

\subsection{Memory and training time analysis on different GPU models.}
In the section Experiment, we point out that the reversible architectures need roughly 33\% more computations than vanilla networks during training. In Table~\ref{analysis}, we further analyzed the memory benefits and compute overhead of our reversible SNNs under popular GPU models in the deep learning literature. We adopt maximum batch size, maximum time step, and allocated memory to demonstrate the memory benefits. As for the compute overhead, training time is adopted as the metric. As is seen, on GPUs with different memory, our RevSResNet37 can achieve $2.69-2.89\times$ maximum batch size and $2.79-3.25 \times$ maximum time steps with roughly $33\%$ more training time. Thus, with the fixed and proved $33\%$ compute burden, the memory reduction brought by our RevSNN is stable and can be generalized to different GPU architectures.

\begin{figure}[htp]
\centering
\includegraphics[width=1.\columnwidth]{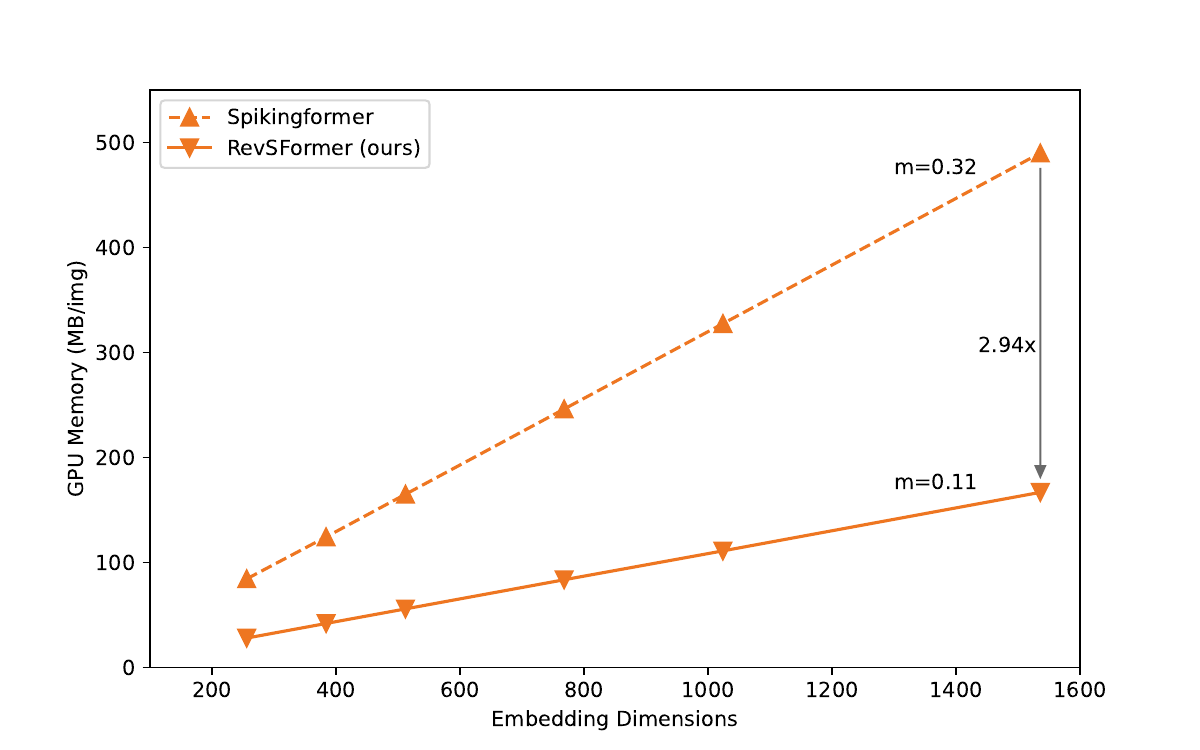} 
\caption{\small{The relationship between GPU memory and embedding dimension.}
}
\label{m_vs_dim}
\end{figure}

\end{document}